%% file: main.tex
\documentclass{llncs}

\usepackage{amssymb,amsmath}

\usepackage{graphicx} % more modern

\usepackage[numbers]{natbib}

\usepackage{algorithm,algorithmic}

\newcommand{\R}{\mathbb{R}}
\newcommand{\grad}{\nabla}
\newcommand{\norm}[1]{{\left\|{#1}\right\|}}

\newcommand{\Breg}{\mathcal{B}}

\DeclareMathOperator{\Regret}{Regret}
\DeclareMathOperator*{\Exp}{\mathbf{E}}
\DeclareMathOperator*{\argmin}{argmin}

\title{Scale-Free Algorithms for Online Linear Optimization}

\author{Francesco Orabona \and D\'avid P\'al}
\institute{Yahoo Labs, 11th Floor, 229 West 43rd Street, New York, NY 10036, USA\\
\email{francesco@orabona.com} and \email{dpal@yahoo-inc.com}}

\begin{document}

\maketitle

\begin{abstract}
\input{abstract}
\end{abstract}

\input{introduction}
\input{preliminaries}
\input{ada-ftrl}
\input{solo-ftrl}

\input{lower-bound}
\input{per-coordinate}

{\small
\bibliographystyle{plain}
\bibliography{biblio}
}

\appendix
\input{preliminaries-proofs}
\input{limits}
\input{ada-ftrl-proofs}

\input{solo-ftrl-proofs}
\input{lower-bound-proof}

\end{document}

%% file: abstract.tex
We design algorithms for online linear optimization that have optimal regret
and at the same time do not need to know any upper or lower bounds on the norm
of the loss vectors.  We achieve adaptiveness to norms of loss vectors by scale
invariance, i.e., our algorithms make exactly the same decisions if the
sequence of loss vectors is multiplied by any positive constant.  Our
algorithms work for any decision set, bounded or unbounded.  For unbounded
decisions sets, these are the first truly adaptive algorithms for online linear
optimization.

%% file: introduction.tex
\section{Introduction}
\label{section:introduction}

Online Linear Optimization (OLO) is a problem where an algorithm repeatedly
chooses a point $w_t$ from a convex decision set $K$, observes an arbitrary, or
even adversarially chosen, loss vector $\ell_t$ and suffers loss $\langle
\ell_t, w_t \rangle$.  The goal of the algorithm is to have a small cumulative
loss. Performance of an algorithm is evaluated by the so-called regret, which
is the difference of cumulative losses of the algorithm and of the
(hypothetical) strategy that would choose in every round the same best point in
hindsight.

OLO is a fundamental problem in machine
learning~\citep{Cesa-Bianchi-Lugosi-2006, Shalev-Shwartz-2011}.  Many learning
problems can be directly phrased as OLO, e.g., learning with expert
advice~\citep{Littlestone-Warmuth-1994, Vovk-1998,
Cesa-Bianchi-Haussler-Helmbold-Schapire-Warmuth-1997}, online combinatorial
optimization~\citep{Koolen-Warmuth-Kivinen-2010}. Other
problems can be reduced to OLO, e.g. online convex
optimization~\citep[Chapter~2]{Shalev-Shwartz-2011}, online classification and
regression~\citep[Chapters~11~and~12]{Cesa-Bianchi-Lugosi-2006}, multi-armed
problems~\citep[Chapter~6]{Cesa-Bianchi-Lugosi-2006}, and batch and stochastic
optimization of convex functions~\citep{NemirovskyY83}.  Hence, a result in OLO
immediately implies other results in all these domains.

The adversarial choice of the loss vectors received by the algorithm is what
makes the OLO problem challenging. In particular, if an OLO algorithm commits
to an upper bound on the norm of future loss vectors, its regret can be made
arbitrarily large through an adversarial strategy that produces loss vectors
with norms that exceed the upper bound.

For this reason, most of the existing OLO algorithms receive as an input---or
explicitly assume---an upper bound $B$ on the norm of the loss vectors.  The
input $B$ is often disguised as the learning rate, the regularization
parameter, or the parameter of strong convexity of the regularizer. Examples of
such algorithms include the Hedge algorithm or online projected gradient
descent with fixed learning rate.  However, these algorithms have two obvious
drawbacks.

First, they do not come with any regret guarantee for sequences of loss vectors
with norms exceeding $B$. Second, on sequences where the norm of loss vectors
is bounded by $b \ll B$, these algorithms fail to have an optimal regret
guarantee that depends on $b$ rather than on $B$.

% Change spacing between rows.
\renewcommand{\arraystretch}{1.5}

\begin{table}[t]
% Change these numbers to change font size.
\fontsize{8}{8.2}\selectfont
\centering
\begin{tabular}{|p{3.6cm}|c|p{3.4cm}|c|}
\hline
\textbf{Algorithm} & \textbf{Decisions Set(s)} & \textbf{Regularizer(s)} & \textbf{Scale-Free} \\ \hline \hline
\textsc{Hedge} \cite{Freund-Schapire-1999} & Probability Simplex & Negative Entropy & No \\ \hline
\textsc{GIGA} \cite{Zinkevich-2003} & Any Bounded & $\frac{1}{2}\|w\|_2^2$ & No \\ \hline
\textsc{RDA} \cite{Xiao10} & \textbf{Any} & \textbf{Any Strongly Convex} & No \\ \hline
\textsc{FTRL-Proximal} \cite{McMahanS10,McMahan14} & Any Bounded & $\frac{1}{2}\|w\|_2^2 + $ any convex func. & \textbf{Yes} \\ \hline
\textsc{AdaGrad MD} \cite{Duchi-Hazan-Singer-2011} & Any Bounded & $\frac{1}{2}\|w\|_2^2 + $ any convex func. & \textbf{Yes} \\ \hline
\textsc{AdaGrad FTRL} \cite{Duchi-Hazan-Singer-2011} & \textbf{Any} & $\frac{1}{2}\|w\|_2^2 + $ any convex func. & No \\ \hline
\textsc{AdaHedge} \cite{de-Rooij-van-Erven-Grunwald-Koolen-2014} & Probability Simplex & Negative Entropy & \textbf{Yes} \\ \hline
\textsc{Optimistic MD} \cite{RakhlinK13} & $\sup_{u,v \in K} \Breg_f(u,v) < \infty$ & \textbf{Any Strongly Convex} & \textbf{Yes} \\ \hline
\textsc{NAG} \cite{RossML13} & $\{u: \max_t \langle \ell_t, u\rangle \leq C\}$ & $\frac{1}{2}\|w\|_2^2 $& Partially\footnotemark \\ \hline
\textsc{Scale invariant algorithms} \cite{OrabonaCCB14} & \textbf{Any} & $\frac{1}{2}\|w\|_p^2 + $ any convex func. \newline $1<p\leq2$ & Partially\textsuperscript{\ref{footnote-label}} \\ \hline
\textsc{AdaFTRL} \textbf{[this paper]} & Any Bounded & \textbf{Any Strongly Convex} & \textbf{Yes} \\ \hline
\textsc{SOLO FTRL} \textbf{[this paper]} & \textbf{Any} & \textbf{Any Strongly Convex} & \textbf{Yes} \\ \hline
\end{tabular}
\caption{Selected results for OLO. Best results in each column are in bold.
\label{table:results}
}
\vspace{-0.5cm}
\end{table}
\footnotetext{\label{footnote-label} These algorithms attempt to produce an invariant sequence of predictions $\langle w_t, \ell_t \rangle$, rather than a sequence of invariant $w_t$.}

There is a clear practical need to design algorithms that adapt automatically
to norms of the loss vectors.  A natural, yet overlooked, design method to
achieve this type of adaptivity is by insisting to have a \textbf{scale-free}
algorithm.  That is, the sequence of decisions of the algorithm does not change
if the sequence of loss vectors is multiplied by a positive constant.

A summary of algorithms for OLO is presented in Table~\ref{table:results}.
While the scale-free property has been looked at in the expert setting, in the
general OLO setting this issue has been largely ignored.  In particular, the
AdaHedge~\citep{de-Rooij-van-Erven-Grunwald-Koolen-2014} algorithm, for
prediction with expert advice, is specifically designed to be scale-free.  A
notable exception in the OLO literature is the discussion of the ``off-by-one''
issue in~\cite{McMahan14}, where it is explained that even the popular AdaGrad
algorithm~\cite{Duchi-Hazan-Singer-2011} is not completely adaptive; see also
our discussion in Section~\ref{section:solo-ftrl}. In particular, existing
scale-free algorithms cover only some norms/regularizers and \emph{only}
bounded decision sets. The case of \textbf{unbounded decision sets},
practically the most interesting one for machine learning applications, remains
completely unsolved.

Rather than trying to design strategies for a particular form of loss
vectors and/or decision sets, in this paper we explicitly focus on the
scale-free property. Regret of scale-free algorithms is proportional to the
scale of the losses, ensuring optimal linear dependency on the maximum norm of
the loss vectors.

The contribution of this paper is twofold. First, in
Section~\ref{section:ada-ftrl} we show that the analysis and design of
AdaHedge can be generalized to the OLO scenario and to any strongly convex
regularizer, in an algorithm we call \textsc{AdaFTRL}, providing a new and
rather interesting way to adapt the learning rates to have scale-free
algorithms.  Second, in Section~\ref{section:solo-ftrl} we propose a new and
simple algorithm, \textsc{SOLO FTRL}, that is scale-free and is the
\textbf{first} scale-free online algorithm for unbounded sets with a
non-vacuous regret bound.  Both algorithms are instances of Follow The
Regularized Leader (FTRL) with an adaptive learning rate.  Moreover, our
algorithms show that scale-free algorithms can be obtained in a ``native'' and
simple way, i.e.  without using ``doubling tricks'' that attempt to fix poorly
designed algorithms rather than directly solving the problem.

For both algorithms, we prove that for bounded decision sets the regret after
$T$ rounds is at most $O(\sqrt{\sum_{t=1}^T\|\ell_t\|_*^2})$.  We show that the
$\sqrt{\sum_{t=1}^T \|\ell_t\|_*^2}$ term is necessary by proving a $\Omega (D
\sqrt{\sum_{t=1}^T\|\ell_t\|_*^2} )$ lower bound on the regret of any algorithm
for OLO for any decision set with diameter $D$ with respect to the primal norm
$\|\cdot\|$. For the \textsc{SOLO FTRL} algorithm, we prove an
$O(\max_{t=1,2,\dots,T} \|\ell_t\|_* \sqrt{T})$ regret bound for any unbounded
decision set.

Our algorithms are also \textbf{any-time}, i.e., do not need to know the
number of rounds in advance and our regret bounds hold for all time steps
simultaneously.

%% file: preliminaries.tex
\section{Notation and Preliminaries}
\label{section:preliminaries}

Let $V$ be a finite-dimensional real vector space equipped with a norm
$\|\cdot\|$. We denote by $V^*$ its dual vector space.  The bi-linear map
associated with $(V^*, V)$ is denoted by $\langle \cdot, \cdot \rangle:V^*
\times V \to \R$.  The dual norm of $\|\cdot\|$ is $\|\cdot\|_*$.

In OLO, in each round $t=1,2,\dots$, the algorithm chooses a point $w_t$ in the
decision set $K \subseteq V$ and then the algorithm observes a loss vector
$\ell_t \in V^*$. The instantaneous loss of the algorithm in round $t$ is
$\langle \ell_t, w_t \rangle$. The cumulative loss of the algorithm after $T$
rounds is $\sum_{t=1}^T \langle \ell_t, w_t \rangle$.  The regret of the
algorithm with respect to a point $u \in K$ is
$$
\Regret_T(u) = \sum_{t=1}^T \langle \ell_t, w_t \rangle - \sum_{t=1}^T \langle \ell_t, u \rangle,
$$
and the regret with respect to the best point is $\Regret_T= \sup_{u \in K}
\Regret_T(u)$.  We assume that $K$ is a non-empty closed convex subset of $V$.
Sometimes we will assume that $K$ is also bounded. We denote by $D$ its
diameter with respect to $\|\cdot\|$, i.e. $D = \sup_{u,v \in K} \|u - v\|$. If
$K$ is unbounded, $D = +\infty$.

\textbf{Convex Analysis.} The \emph{Bregman divergence} of a convex
differentiable function $f$ is defined as $\Breg_f(u,v) = f(u) - f(v) - \langle
\grad f(v), u - v \rangle$.  Note that $\Breg_f(u,v) \ge 0$ for any $u,v$ which
follows directly from the definition of convexity of $f$.

The \emph{Fenchel conjugate} of a function $f:K \to \R$ is the function
$f^*:V^* \to \R \cup \{+\infty\}$ defined as $f^*(\ell) = \sup_{w \in K} \left(
\langle \ell, w \rangle - f(w) \right)$.  The Fenchel conjugate of any function
is convex (since it is a supremum of affine functions) and satisfies for all $w
\in K$ and all $\ell \in V^*$  the \emph{Fenchel-Young inequality}
$
f(w) + f^*(\ell) \ge \langle \ell, w \rangle
$.

Monotonicity of Fenchel conjugates follows easily from the definition: If
$f,g:K \to \R$ satisfy $f(w) \le g(w)$ for all $w \in K$ then $f^*(\ell) \ge
g^*(\ell)$ for every $\ell \in V^*$.

Given $\lambda > 0$, a function $f:K \to \R$ is called \emph{$\lambda$-strongly convex}
with respect to a norm $\|\cdot\|$ if and only if, for all $x,y \in K$,
$$
f(y) \ge f(x) + \langle \grad f(x), y - x \rangle + \frac{\lambda}{2}\|x - y\|^2 \; ,
$$
where $\grad f(x)$ is any subgradient of $f$ at point $x$.

The following proposition relates the range of values of a strongly convex
function to the diameter of its domain. The proof can be found in
Appendix~\ref{section:definitions-proofs}.

\begin{proposition}[Diameter vs. Range]
\label{proposition:diameter-vs-range}
Let $K \subseteq V$ be a non-empty bounded closed convex subset.  Let $D =
\sup_{u,v \in K} \|u - v\|$ be its diameter with respect to $\|\cdot\|$.  Let
$f:K \to \R$ be a non-negative lower semi-continuous function that is
$1$-strongly convex with respect to $\|\cdot\|$.  Then, $D \le \sqrt{8 \sup_{v
\in K} f(v)}$.
\end{proposition}

Fenchel conjugates and strongly convex functions have certain nice properties,
which we list in Proposition~\ref{proposition:conjugate-properties} below.

\begin{proposition}[Fenchel Conjugates of Strongly Convex Functions]
\label{proposition:conjugate-properties}
Let $K \subseteq V$ be a non-empty closed convex set with diameter $D:=\sup_{u,v \in K} \|u-v\|$.
Let $\lambda > 0$, and let $f:K \to \R$ be a lower semi-continuous function
that is $\lambda$-strongly convex with respect to $\|\cdot\|$.
The Fenchel conjugate of $f$ satisfies:
\begin{enumerate}
\item $f^*$ is finite everywhere and differentiable.
\item $\grad f^*(\ell) = \argmin_{w \in K} \left( f(w) - \langle \ell, w \rangle \right)$
\item For any $\ell \in V^*$,
$f^*(\ell) + f(\grad f^*(\ell)) = \langle \ell, \grad f^*(\ell) \rangle$.
\item $f^*$ is $\frac{1}{\lambda}$-strongly smooth i.e. for any $x,y \in V^*$,
$\Breg_{f^*}(x, y) \le \frac{1}{2\lambda} \|x - y\|_*^2$.

\item $f^*$ has $\frac{1}{\lambda}$-Lipschitz continuous gradients i.e.
$\|\grad f^*(x) - \grad f^*(y)\| \le \frac{1}{\lambda} \|x - y\|_*$
for any $x,y \in V^*$.

\item $\Breg_{f^*}(x,y) \le D\|x-y\|_*$ for any $x,y \in V^*$.

\item $\|\grad f^*(x) - \grad f^*(y)\| \le D$ for any $x,y \in V^*$.

\item For any $c > 0$, $(cf(\cdot))^* = cf^*(\cdot/c)$.
\end{enumerate}
\end{proposition}

Except for properties 6 and 7, the proofs can be found
in~\cite{Shalev-Shwartz-2007}.  Property 6 is proven in
Appendix~\ref{section:definitions-proofs}. Property 7 trivially follows from
property 2.

%\label{section:generic-ftrl}

\begin{algorithm}[t]
\caption{\textsc{FTRL with Varying Regularizer}}
\label{algorithm:ftrl-varying-regularizer}
\begin{algorithmic}[1]
\REQUIRE Sequence of regularizers $\{R_t\}_{t=1}^\infty$
\STATE Initialize $L_0 \leftarrow 0$
\FOR{$t=1,2,3,\dots$}
\STATE $w_t \leftarrow \argmin_{w \in K} \left( \langle L_{t-1}, w \rangle + R_t(w) \right)$
\STATE Predict $w_t$
\STATE Observe $\ell_t \in V^*$
\STATE $L_t \leftarrow L_{t-1} + \ell_t$
\ENDFOR
\end{algorithmic}
\end{algorithm}

\textbf{Generic FTRL with Varying Regularizer.}
Our scale-free online learning algorithms are versions of the \textsc{Follow
The Regularized Leader} (FTRL) algorithm with varying regularizers, presented
as Algorithm~\ref{algorithm:ftrl-varying-regularizer}.  The following lemma
bounds its regret.

\begin{lemma}[Lemma 1 in \cite{OrabonaCCB14}]
\label{lemma:generic-regret-bound}
For any sequence $\{R_t\}_{t=1}^\infty$ of strongly convex lower
semi-continuous regularizers, regret of
Algorithm~\ref{algorithm:ftrl-varying-regularizer} is upper
bounded as
$$
\Regret_T(u) \le R_{T+1}(u) + R_1^*(0) + \sum_{t=1}^{T} \Breg_{R_t^*}(-L_t, -L_{t-1}) - R_t^*(-L_t) + R_{t+1}^*(-L_t) \; .
$$
\end{lemma}
The lemma allows data dependent regularizers. That is, $R_t$ can depend on the past loss
vectors $\ell_1, \ell_2, \dots, \ell_{t-1}$.

%% file: ada-ftrl.tex
\section{AdaFTRL}
\label{section:ada-ftrl}

In this section we generalize the AdaHedge
algorithm~\cite{de-Rooij-van-Erven-Grunwald-Koolen-2014} to the OLO setting,
showing that it retains its scale-free property. The analysis is very general
and based on general properties of strongly convex functions, rather than
specific properties of the entropic regularizer like in AdaHedge.

Assume that $K$ is bounded and that $R(w)$ is a strongly convex lower
semi-continuous function bounded from above.  We instantiate
Algorithm~\ref{algorithm:ftrl-varying-regularizer} with the sequence of
regularizers
\begin{equation}
\label{equation:ada-ftrl}
R_t(w) = \Delta_{t-1} R(w)
\quad \text{where}
\quad \Delta_{t}=\sum_{i=1}^{t} \Delta_{i-1} \Breg_{R^*}\left(- \frac{L_i}{\Delta_{i-1}}, -\frac{L_{i-1}}{\Delta_{i-1}}\right) \; .
\end{equation}

The sequence $\{\Delta_t\}_{t=0}^\infty$ is non-negative and non-decreasing.
Also, $\Delta_t$ as a function of $\{\ell_s\}_{s=1}^t$ is positive homogenous
of degree one, making the algorithm scale-free.

If $\Delta_{i-1} = 0$, we define 
$\Delta_{i-1} \Breg_{R^*}(\frac{-L_i}{\Delta_{i-1}}, \frac{-L_{i-1}}{\Delta_{i-1}})$ 
as $\lim_{a \to 0^+} a \Breg_{R^*}(\frac{-L_i}{a}, \frac{-L_{i-1}}{a})$ which always exists and is
finite; see Appendix~\ref{section:limits}.  Similarly, when $\Delta_{t-1} = 0$, we
define $w_t = \argmin_{w \in K} \langle L_{t-1}, w \rangle$ where ties among
minimizers are broken by taking the one with the smallest value of $R(w)$,
which is unique due to strong convexity; this is the same as $w_t = \lim_{a \to
0^+} \argmin_{w \in K} (\langle L_{t-1}, w \rangle + aR(w))$.

Our main result is an $O(\sqrt{\sum_{t=1}^T \|\ell_t\|_*^2})$ upper bound on
the regret of the algorithm after $T$ rounds, without the need to know before
hand an upper bound on $\|\ell_t\|_*$.  We prove the theorem in
Section~\ref{section:ada-ftrl-regret-bound}.

\begin{theorem}[Regret Bound]
\label{theorem:ada-ftrl-regret-bound}
Suppose $K \subseteq V$ is a non-empty bounded closed convex subset. Let $D =
\sup_{x,y \in K} \|x - y\|$ be its diameter with respect to a norm $\|\cdot\|$.
Suppose that the regularizer $R:K \to \R$ is a non-negative lower
semi-continuous function that is $\lambda$-strongly convex with respect to
$\|\cdot\|$ and is bounded from above.  The regret of AdaFTRL satisfies
$$
\Regret_T(u) \le \sqrt{3} \max\left\{D, \frac{1}{\sqrt{2\lambda}} \right\} \sqrt{\sum_{t=1}^T \|\ell_t\|_*^2} \left(1 + R(u) \right) \; .
$$
\end{theorem}

The regret bound can be optimized by choosing the optimal multiple of the
regularizer.  Namely, we choose regularizer of the form $\lambda f(w)$ where
$f(w)$ is $1$-strongly convex and optimize over $\lambda$. The result of the
optimization is the following corollary.  Its proof can be found in
Appendix~\ref{section:ada-ftrl-proof}.

\begin{corollary}[Regret Bound]
\label{corollary:ada-ftrl-regret-bound}
Suppose $K \subseteq V$ is a non-empty bounded closed convex subset. Suppose
$f:K \to \R$ is a non-negative lower semi-continuous function that is
$1$-strongly convex with respect to $\|\cdot\|$ and is bounded from above.  The
regret of AdaFTRL with regularizer
$$
R(w) = \frac{f(w)}{16 \cdot \sup_{v \in K} f(v)}
\qquad \text{satisfies} \qquad
\Regret_T \le
5.3 \sqrt{\sup_{v \in K} f(v) \sum_{t=1}^T \|\ell_t\|_*^2} \; .
$$
\end{corollary}

\subsection{Proof of Regret Bound for AdaFTRL}
\label{section:ada-ftrl-regret-bound}

\begin{lemma}[Initial Regret Bound]
\label{lemma:initial-regret-bound}
AdaFTRL, for any $u \in K$ and any $u \ge 0$, satisfies
$\Regret_T(u) \le \left(1 + R(u) \right) \Delta_T$.
\end{lemma}

\begin{proof}
Let $R_t(w) = \Delta_{t-1} R(w)$. Since $R$ is non-negative,
$\{R_t\}_{t=1}^\infty$ is non-decreasing.  Hence, $R_t^*(\ell) \ge
R_{t+1}^*(\ell)$ for every $\ell \in V^*$ and thus $R_t^*(-L_t) -
R_{t+1}^*(-L_t) \ge 0$.  So, by Lemma~\ref{lemma:generic-regret-bound},
\begin{equation}
\label{equation:regret-bound-inequality}
\Regret_T(u) \le R_{T+1}(u) + R_1^*(0) + \sum_{t=1}^{T} \Breg_{R_t^*}(-L_t, -L_{t-1}) \; .
\end{equation}
Since, $\Breg_{R_t^*}(u,v) = \Delta_{t-1} \Breg_{R^*}(\frac{u}{\Delta_{t-1}},
\frac{v}{\Delta_{t-1}})$ by definition of Bregman divergence and Part 8 of
Proposition~\ref{proposition:conjugate-properties}, we have $\sum_{t=1}^T
\Breg_{R_t^*}(-L_t, -L_{t-1}) = \Delta_T$.
\end{proof}

\begin{lemma}[Recurrence]
\label{lemma:gap-recurrence}
Let $D = \sup_{u, v \in K} \|u -v\|$ be the diameter of $K$.
The sequence $\{\Delta_t\}_{t=1}^\infty$ generated by AdaFTRL satisfies for any $t \ge 1$,
$$
\Delta_t \le \Delta_{t-1} + \min \left\{ D\|\ell_t\|_*, \ \frac{\|\ell_t\|_*^2}{2\lambda \Delta_{t-1}} \right\} \; .
$$
\end{lemma}

\begin{proof}
The inequality results from strong convexity of $R_t(w)$ and
Proposition~\ref{proposition:conjugate-properties}.
\end{proof}

\begin{lemma}[Solution of the Recurrence]
\label{lemma:recurrence-solution}
Let $D$ be the diameter of $K$. The sequence $\{\Delta_t\}_{t=0}^\infty$
generated by AdaFTRL satisfies for any $T \ge 0$,
$$
\Delta_T \le \sqrt{3} \max\left\{D, \frac{1}{\sqrt{2\lambda}} \right\} \sqrt{\sum_{t=1}^T \|\ell_t\|_*^2} \; .
$$
\end{lemma}
Proof of the Lemma~\ref{lemma:recurrence-solution} is deferred to
Appendix~\ref{section:ada-ftrl-proof}.
Theorem~\ref{theorem:ada-ftrl-regret-bound} follows from
Lemmas~\ref{lemma:initial-regret-bound}~and~\ref{lemma:recurrence-solution}.

%% file: solo-ftrl.tex
\section{SOLO FTRL}
\label{section:solo-ftrl}

The closest algorithm to a scale-free one in the OLO literature is the AdaGrad
algorithm~\cite{Duchi-Hazan-Singer-2011}.  It uses a regularizer on each
coordinate of the form
\begin{equation*}
R_t(w) = R(w) \left(\delta + \sqrt{\sum_{s=1}^{t-1} \|\ell_s\|_*^2} \right).
\end{equation*}
This kind of regularizer would yield a scale-free algorithm \emph{only} for
$\delta=0$.  Unfortunately, the regret bound in~\cite{Duchi-Hazan-Singer-2011}
becomes vacuous for such setting in the unbounded case. In fact, it requires
$\delta$ to be greater than $\|\ell_t\|_*$ for all time steps $t$, requiring
knowledge of the future (see Theorem~5 in~\cite{Duchi-Hazan-Singer-2011}). In
other words, despite of its name, AdaGrad is not fully adaptive to the norm of
the loss vectors. Identical considerations hold for the FTRL-Proximal in
\cite{McMahanS10,McMahan14}: the scale-free setting of the learning rate is
valid only in the bounded case.

One simple approach would be to use a doubling trick on $\delta$ in order to
estimate on the fly the maximum norm of the losses. Note that a naive strategy
would still fail because the initial value of $\delta$ should be data-dependent
in order to have a scale-free algorithm. Moreover, we would have to upper bound
the regret in all the rounds where the norm of the current loss is bigger than
the estimate. Finally, the algorithm would depend on an additional parameter,
the ``doubling'' power. Hence, even guaranteeing a regret bound\footnote{For
lack of space, we cannot include the regret bound for the doubling trick
version. It would be exactly the same as in
Theorem~\ref{theorem:regret-solo-ftrl}, following a similar analysis, but with
the additional parameter of the doubling power.}, such strategy would give the
feeling that FTRL needs to be ``fixed'' in order to obtain a scale-free
algorithm.

In the following, we propose a much simpler and better approach.  We propose to
use Algorithm~\ref{algorithm:ftrl-varying-regularizer} with the regularizer
\begin{equation*}
R_t(w) = R(w) \sqrt{\sum_{s=1}^{t-1} \|\ell_s\|_*^2} \; ,
\end{equation*}
where $R:K \to \R$ is any strongly convex function. Through a refined analysis,
we show that the regularizer suffices to obtain an optimal regret bound for any
decision set, bounded or unbounded.  We call such variant \textsc{Scale-free
Online Linear Optimization FTRL} algorithm (\textsc{SOLO FTRL}).  Our main
result is the following Theorem, which is proven in
Section~\ref{section:solo-ftrl-regret-bound}.

\begin{theorem}[Regret of \textsc{SOLO FTRL}]
\label{theorem:regret-solo-ftrl}
Suppose $K \subseteq V$ is a non-empty closed convex subset.  Let $D =
\sup_{u,v \in K} \|u - v\|$ be its diameter with respect to a norm $\|\cdot\|$.
Suppose that the regularizer $R:K \to \R$ is a non-negative lower
semi-continuous function that is $\lambda$-strongly convex with respect to
$\|\cdot\|$. The regret of SOLO FTRL satisfies
\begin{align*}
\Regret_T(u)
& \le \left( R(u) + \frac{2.75}{\lambda}\right) \sqrt{\sum_{t=1}^{T} \norm{\ell_t}_*^2}
+ 3.5 \min\left\{\frac{\sqrt{T-1}}{\lambda} , D\right\} \max_{t \le T} \|\ell_t\|_*.
\end{align*}
\end{theorem}

When $K$ is bounded, we can choose the optimal multiple of the regularizer.  We
choose $R(w) = \lambda f(w)$ where $f$ is a $1$-strongly convex function and
optimize $\lambda$.  The result of the optimization is
Corollary~\ref{corollary:regret-solo-ftrl-bounded-set}; the proof is in
Appendix~\ref{section:solo-ftrl-proof}.  It is similar to
Corollary~\ref{corollary:ada-ftrl-regret-bound} for AdaFTRL. The scaling
however is different in the two corollaries.  In
Corollary~\ref{corollary:ada-ftrl-regret-bound}, $\lambda \sim 1/(\sup_{v \in
K} f(v))$ while in Corollary~\ref{corollary:regret-solo-ftrl-bounded-set} we
have $\lambda \sim 1/\sqrt{\sup_{v \in K} f(v)}$.

\begin{corollary}[Regret Bound for Bounded Decision Sets]
\label{corollary:regret-solo-ftrl-bounded-set}
Suppose $K \subseteq V$ is a non-empty bounded closed convex subset.  Suppose
that $f:K \to \R$ is a non-negative lower semi-continuous function that is $1$-strongly
convex with respect to $\|\cdot\|$. SOLO FTRL with regularizer
$$
R(w) = \frac{f(w)\sqrt{2.75}}{\sqrt{\sup_{v \in K} f(v)}}
\quad \text{satisfies} \quad %
\Regret_T \le 13.3 \sqrt{\sup_{v \in K} f(v) \sum_{t=1}^{T} \norm{\ell_t}_*^2} \; .
$$
\end{corollary}

\subsection{Proof of Regret Bound for SOLO FTRL}
\label{section:solo-ftrl-regret-bound}

The proof of Theorem~\ref{theorem:regret-solo-ftrl} relies on an inequality
(Lemma~\ref{lemma:useful}).  Related and weaker inequalities were proved by
~\cite{Auer-Cesa-Bianchi-Gentile-2002} and ~\cite{Jaksch-Ortner-Auer-2010}.
The main property of this inequality is that on the right-hand side $C$ does
\emph{not} multiply the $\sqrt{\sum_{t=1}^T a_t^2}$ term.  We will also use the
well-known technical Lemma~\ref{lemma:sum-of-square-roots-inverses}.

\begin{lemma}[Useful Inequality]
\label{lemma:useful}
Let $C, a_1, a_2, \dots, a_T\geq0$. Then,
$$
\sum_{t=1}^T \min \left\{ a_t^2 / \sqrt{\sum_{s=1}^{t-1} a_s^2}, \ C a_t \right\}
\le 3.5 C \max_{t=1,2,\dots,T} a_t + 3.5 \sqrt{\sum_{t=1}^T a_t^2} \; .
$$
\end{lemma}
\begin{proof}
Without loss of generality, we can assume that $a_t > 0$ for all $t$. Since otherwise we
can remove all $a_t = 0$ without affecting either side of the inequality. Let $M_t = \max\{a_1, a_2, \dots, a_t\}$ and $M_0 = 0$.
We prove that for any $\alpha > 1$
$$
\min\left\{ \frac{a_t^2}{\sqrt{\sum_{s=1}^{t-1} a_s^2}}, C a_t \right\}
\le 2 \sqrt{1+\alpha^2} \left( \sqrt{\sum_{s=1}^t a_s^2} - \sqrt{\sum_{s=1}^{t-1} a_s^2} \right) + \frac{C\alpha( M_t  - M_{t-1})}{\alpha - 1}
$$
from which the inequality follows by summing over $t=1,2,\dots,T$ and choosing $\alpha = \sqrt{2}$.
The inequality follows by case analysis. If $a_t^2 \le \alpha^2 \sum_{s=1}^{t-1} a_s^2$, we have
\begin{multline*}
\min\left\{ \frac{a_t^2}{\sqrt{\sum_{s=1}^{t-1} a_s^2}}, C a_t \right\}
\le \frac{a_t^2}{\sqrt{\sum_{s=1}^{t-1} a_s^2}}
= \frac{a_t^2}{\sqrt{\frac{1}{1+\alpha^2} \left( \alpha^2 \sum_{s=1}^{t-1} a_s^2 + \sum_{s=1}^{t-1} a_s^2 \right)}} \\
\le \frac{a_t^2\sqrt{1+\alpha^2}}{\sqrt{ a_t^2 + \sum_{s=1}^{t-1} a_s^2 }}
= \frac{a_t^2\sqrt{1+\alpha^2}}{\sqrt{\sum_{s=1}^t a_s^2}}
\le 2\sqrt{1+\alpha^2} \left( \sqrt{\sum_{s=1}^t a_s^2} - \sqrt{\sum_{s=1}^{t-1} a_s^2} \right)
\end{multline*}
where we have used $x^2/\sqrt{x^2+y^2} \le 2(\sqrt{x^2+y^2} - \sqrt{y^2})$ in the last step.
On the other hand, if $a_t^2 > \alpha^2 \sum_{t=1}^{t-1} a_s^2$, we have
\begin{multline*}
\min\left\{ \frac{a_t^2}{\sqrt{\sum_{s=1}^{t-1} a_s^2}}, C a_t \right\}
\le C a_t
= C \frac{\alpha a_t  - a_t}{\alpha - 1}
\le \frac{C}{\alpha - 1} \left( \alpha a_t  - \alpha \sqrt{\sum_{s=1}^{t-1} a_s^2} \right) \\
= \frac{C\alpha}{\alpha - 1} \left( a_t  - \sqrt{\sum_{s=1}^{t-1} a_s^2} \right)
\le \frac{C\alpha}{\alpha - 1} \left( a_t  - M_{t-1} \right)
= \frac{C\alpha}{\alpha - 1} \left( M_t  - M_{t-1} \right)
\end{multline*}
where we have used that $a_t = M_t$ and $\sqrt{\sum_{s=1}^{t-1} a_s^2} \ge M_{t-1}$.
\end{proof}

\begin{lemma}[Lemma~3.5 in \cite{Auer-Cesa-Bianchi-Gentile-2002}]
\label{lemma:sum-of-square-roots-inverses}
Let $a_1, a_2, \dots, a_T$ be non-negative real numbers. If $a_1 > 0$ then,
$$
\sum_{t=1}^T a_t / \sqrt{\sum_{s=1}^t a_s} \le 2 \sqrt{\sum_{t=1}^T a_t} \; .
$$
\end{lemma}

\begin{proof}[Proof of Theorem~\ref{theorem:regret-solo-ftrl}]
Let $\eta_t=\tfrac{1}{\sqrt{\sum_{s=1}^{t-1} \|\ell_s\|_*^2}}$, hence $R_t(w) = \tfrac{1}{\eta_t} R(w)$.
We assume without loss of
generality that $\|\ell_t\|_* > 0$ for all $t$, since otherwise we can remove
all rounds $t$ where $\ell_t = 0$ without affecting regret and the
predictions of the algorithm on the remaining rounds.
By Lemma~\ref{lemma:generic-regret-bound},
\begin{align*}
\Regret_T(u)
& \le \frac{1}{\eta_{T+1}} R(u) + \sum_{t=1}^T \left( \Breg_{R_t^*}(-L_t, -L_{t-1}) - R_t^*(-L_t) + R_{t+1}^*(-L_t) \right) \; .
\end{align*}
We upper bound the terms of the sum in two different ways.
First, by Proposition~\ref{proposition:conjugate-properties}, we have
$$
\Breg_{R_t^*}(-L_t, -L_{t-1}) - R_t^*(-L_t) + R_{t+1}^*(-L_t)
\le \Breg_{R_t^*}(-L_t, -L_{t-1})
\le \frac{\eta_t \|\ell_t\|_*^2}{2\lambda} \; .
$$
Second, we have
\begin{align*}
& \Breg_{R_t^*}(-L_t, -L_{t-1}) - R_t^*(-L_t) + R_{t+1}^*(-L_t) \\
& = \Breg_{R_{t+1}^*}(-L_t, -L_{t-1}) + R^*_{t+1}(-L_{t-1}) - R_t^*(-L_{t-1}) \\
& \qquad + \langle \nabla R_t^*(-L_{t-1})-\nabla R_{t+1}^*(-L_{t-1}), \ell_t \rangle  \\
& \le \tfrac{1}{2\lambda} \eta_{t+1} \|\ell_t\|_*^2 + \| \nabla R_t^*(-L_{t-1})-\nabla R_{t+1}^*(-L_{t-1})\| \cdot \|\ell_t\|_* \\
& = \tfrac{1}{2\lambda} \eta_{t+1} \|\ell_t\|_*^2 + \| \nabla R^*(- \eta_{t} L_{t-1})-\nabla R^*(- \eta_{t+1} L_{t-1})\| \cdot \|\ell_t\|_* \\
& \le \frac{\eta_{t+1} \|\ell_t\|_*^2}{2\lambda} + \min\left\{\frac{1}{\lambda} \|L_{t-1}\|_* \left(\eta_{t} - \eta_{t+1} \right), D\right\} \|\ell_t\|_* \; ,
\end{align*}
where in the first inequality we have used the fact that $R^*_{t+1}(-L_{t-1})
\le R_t^*(-L_{t-1})$, H\"older's inequality, and
Proposition~\ref{proposition:conjugate-properties}.  In the second inequality
we have used properties 5 and 7 of
Proposition~\ref{proposition:conjugate-properties}. Using the definition of
$\eta_{t+1}$ we have
\begin{align*}
\frac{\|L_{t-1}\|_* (\eta_{t} -\eta_{t+1})}{\lambda}
\le \frac{ \|L_{t-1}\|_*}{\lambda \sqrt{\sum_{i=1}^{t-1} \|\ell_i\|_*^2}}
\le \frac{ \sum_{i=1}^{t-1} \|\ell_i\|_*}{\lambda \sqrt{\sum_{i=1}^{t-1} \|\ell_i\|_*^2}}
\le \frac{\sqrt{t-1}}{\lambda}
\le \frac{\sqrt{T-1}}{\lambda}.
\end{align*}
Denoting by $H=\min\left\{\frac{\sqrt{T-1}}{\lambda},D\right\}$ we have
\begin{align*}
& \Regret_T(u)
\le \frac{1}{\eta_{T+1}} R(u) + \sum_{t=1}^T \min\left\{ \frac{\eta_t \|\ell_t\|_*^2}{2\lambda}, \ H \|\ell_t\|_* + \frac{\eta_{t+1} \|\ell_t\|_*^2}{2\lambda}  \right\} \\
& \le \frac{1}{\eta_{T+1}} R(u) + \frac{1}{2\lambda} \sum_{t=1}^T  \eta_{t+1} \|\ell_t\|_*^2 + \frac{1}{2\lambda} \sum_{t=1}^T \min\left\{ \eta_t \|\ell_t\|_*^2, \ 2 \lambda H \|\ell_t\|_* \right\} \\
& = \frac{1}{\eta_{T+1}} R(u) + \frac{1}{2\lambda} \sum_{t=1}^T  \frac{\|\ell_t\|_*^2}{\sqrt{\sum_{s=1}^t \|\ell_t\|_*^2}} + \frac{1}{2 \lambda} \sum_{t=1}^T \min\left\{ \frac{\|\ell_t\|_*^2}{\sqrt{\sum_{s=1}^{t-1} \|\ell_s\|_*^2}}, \ 2 \lambda H \|\ell_t\|_* \right\} \; .
\end{align*}
We bound each of the three terms separately. By definition of $\eta_{T+1}$, the
first term is $\frac{1}{\eta_{T+1}} R(u) = R(u) \sqrt{\sum_{t=1}^T
\|\ell_t\|_*^2}$.  We upper bound the second term using
Lemma~\ref{lemma:sum-of-square-roots-inverses} as
$$
\frac{1}{2\lambda} \sum_{t=1}^T  \frac{\|\ell_t\|_*^2}{\sqrt{\sum_{s=1}^t \|\ell_t\|_*^2}}
\le \frac{1}{\lambda} \sqrt{\sum_{t=1}^T \|\ell_t\|_*^2} \; .
$$
Finally, by Lemma~\ref{lemma:useful} we upper bound the third term as
$$
\frac{1}{2 \lambda} \sum_{t=1}^T \min\left\{ \frac{\|\ell_t\|_*^2}{\sqrt{\sum_{s=1}^{t-1} \|\ell_s\|_*^2}}, \ 2 \lambda \|\ell_t\|_* H \right\}
\le 3.5 H \max_{t \le T} \|\ell_t\|_* + \frac{1.75}{\lambda} \sqrt{\sum_{t=1}^T \|\ell_t\|_*^2} \; .
$$
Putting everything together gives the stated bound.
\end{proof}

%% file: lower-bound.tex
\section{Lower Bound}
\label{section:lower-bound}

We show a lower bound on the worst-case regret of any algorithm for OLO. The
proof is a standard probabilistic argument, which we present in
Appendix~\ref{section:lower-bound-proof}.

\begin{theorem}[Lower Bound]
\label{theorem:simple-lower-bound}
Let $K \subseteq V$ be any non-empty bounded closed convex subset. Let $D =
\sup_{u,v \in K} \|u - v\|$ be the diameter of $K$. Let $A$ be any (possibly
randomized) algorithm for OLO on $K$. Let $T$ be any non-negative integer and
let $a_1, a_2, \dots, a_T$ be any non-negative real numbers.  There exists a
sequence of vectors $\ell_1, \ell_2, \dots, \ell_T$ in the dual vector space
$V^*$ such that $\|\ell_1\|_* = a_1, \|\ell_2\|_* = a_2, \dots, \|\ell_T\|_* =
a_T$ and the regret of algorithm $A$ satisfies
\begin{equation}
\label{equation:simple-lower-bound}
\Regret_T \ge \frac{D}{\sqrt{8}} \sqrt{\sum_{t=1}^T\|\ell_t\|_*^2} \; .
\end{equation}
\end{theorem}

The upper bounds on the regret, which we have proved for our algorithms, have
the same dependency on the norms of loss vectors.  However, a gap remains
between the lower bound and the upper bounds.

Our upper bounds are of the form $O(\sqrt{\sup_{v \in K} f(v) \sum_{t=1}^T
\|\ell_t\|_*^2})$ where $f$ is any $1$-strongly convex function with respect to
$\|\cdot\|$.  The same upper bound is also achieved by FTRL with a constant
learning rate when the number of rounds $T$ and $\sum_{t=1}^T \|\ell_t\|_*^2$
is known upfront \citep[Chapter 2]{Shalev-Shwartz-2011}.  The lower bound is
$\Omega(D\sqrt{\sum_{t=1}^T \|\ell_t\|_*^2})$.

The gap between $D$ and $\sqrt{\sup_{v \in K} f(v)}$ can be substantial.  For
example, if $K$ is the probability simplex in $\R^d$ and $f(w) = \ln(d) +
\sum_{i=1}^d w_i \ln w_i$ is the shifted negative entropy, the
$\|\cdot\|_1$-diameter of $K$ is $2$, $f$ is non-negative and $1$-strongly
convex w.r.t. $\|\cdot\|_1$, but $\sup_{v \in K} f(v) = \ln(d)$.  On the other
hand, if the norm $\|\cdot\|_2 = \sqrt{\langle \cdot, \cdot \rangle}$ arises
from an inner product $\langle \cdot, \cdot \rangle$, the lower bound matches
the upper bounds within a constant factor.  The reason is that for any $K$ with
$\|\cdot\|_2$-diameter $D$, the function $f(w) = \frac{1}{2} \|w - w_0\|_2^2$,
where $w_0$ is an arbitrary point in $K$, is $1$-strongly convex w.r.t.
$\|\cdot\|_2$ and satisfies that $\sqrt{\sup_{v \in K} f(v)} \le D$. This leads
to the following open problem (posed also in~\cite{Kwon-Mertikopoulos-2014}):
\begin{quotation}
\noindent
\emph{Given a bounded convex set $K$ and a norm $\|\cdot\|$, construct a non-negative
function $f:K \to \R$ that is $1$-strongly convex with respect to $\|\cdot\|$
and minimizes $\sup_{v \in K} f(v)$.}
\end{quotation}
As shown in~\cite{Srebro-Sridharan-Tewari-2011}, the existence of $f$ with small
$\sup_{v \in K} f(v)$ is equivalent to the existence of an algorithm for OLO with
$\widetilde O(\sqrt{T \sup_{v \in K} f(v)})$ regret assuming $\|\ell_t\|_* \le 1$.
The $\widetilde O$ notation hides a polylogarithmic factor in $T$.

%% file: per-coordinate.tex
\section{Per-Coordinate Learning}

An interesting class of algorithms proposed in~\citep{McMahanS10} and
\citep{Duchi-Hazan-Singer-2011} are based on the so-called per-coordinate
learning rates.  As shown in \cite{Streeter-McMahan-2010}, our algorithms, or
in fact any algorithm for OLO, can be used with per-coordinate
learning rates as well.

Abstractly, we assume that the decision set is a Cartesian product $K=K_1
\times K_2 \times \dots \times K_d$ of a finite number of convex sets.  On each
factor $K_i$, $i=1,2,\dots,d$, we can run any OLO algorithm separately and we
denote by $\Regret_T^{(i)}(u_i)$ its regret with respect to $u_i \in K_i$. The
overall regret with respect to any $u=(u_1, u_2, \dots, u_d) \in K$ can be
written as
$$
\Regret_T(u) = \sum_{i=1}^d \Regret_T^{(i)}(u_i) \; .
$$
If the algorithm for each factor is scale-free, the overall algorithm is
clearly scale-free as well.  Using \textsc{AdaFTRL} or \textsc{SOLO FTRL} for
each factor $K_i$, we generalize and improve existing regret bounds
\citep{McMahanS10, Duchi-Hazan-Singer-2011} for algorithms with per-coordinate
learning rates.

%% file: preliminaries-proofs.tex
\section{Proofs for Preliminaries}
\label{section:definitions-proofs}

\begin{proof}[Proof of Proposition~\ref{proposition:diameter-vs-range}]
Let $S = \sup_{u \in K} f(u)$ and $v^* = \argmin_{v \in K} f(v)$. The minimizer
$v^*$ is guaranteed to exist by lower semi-continuity of $f$ and compactness of
$K$.  Optimality condition for $v^*$ and $1$-strong convexity of $f$ imply that
for any $u \in K$,
$$
S
\ge f(u) - f(v^*)
\ge f(u) - f(v^*) - \langle \grad f(v^*), u - v^* \rangle
\ge \frac{1}{2}\|u - v^*\|^2 \; .
$$
In other words, $\|u - v^*\| \le \sqrt{2S}$. By triangle inequality,
$$
D = \sup_{u,v \in K} \|u - v\| \le \sup_{u,v \in K} \left( \|u - v^*\| + \|v^ * - v\| \right) \le 2\sqrt{2S} = \sqrt{8S} \; .
$$
\end{proof}

\begin{proof}[Proof of Property 6 of Proposition~\ref{proposition:conjugate-properties}]
To bound $\Breg_{f^*}(x,y)$ we add a non-negative divergence term $\Breg_{f^*}(y,x)$.
\begin{align*}
\Breg_{f^*}(x, y)
& \le \Breg_{f^*}(x,y) + \Breg_{f^*}(y,x)
= \langle x - y, \grad f^*(x) - \grad f^*(y) \rangle \\
& \le \|x - y\|_* \cdot \| \grad f^*(x) - \grad f^*(y) \|
\le D \|x - y\|_* \; ,
\end{align*}
where we have used H\"older's inequality and Part 7 of the Proposition.
\end{proof}

%% file: limits.tex
\section{Limits}
\label{section:limits}

\begin{lemma}
\label{lemma:limit-existence}
Let $K$ be a non-empty bounded closed convex subset of a finite dimensional
normed real vector space $(V, \|\cdot\|)$.  Let $R:K \to \R$ be a strongly
convex lower semi-continuous function bounded from above. Then, for any $x,y
\in V^*$,
$$
\lim_{a \to 0^+} a \Breg_{R^*}(x/a, y/a) = \langle x, u - v \rangle \\
$$
where
\begin{align*}
u = \lim_{a \to 0^+} \argmin_{w \in K} \left( a R(w) - \langle x, w \rangle \right)
\quad \text{and} \quad
v = \lim_{a \to 0^+} \argmin_{w \in K} \left( a R(w) - \langle y, w \rangle \right) \; .
\end{align*}
\end{lemma}

\begin{proof}
Using Part 3 of Proposition~\ref{proposition:conjugate-properties} we can write
the divergence
\begin{align*}
a \Breg_{R^*}(x/a, y/a) & = a R^*(x/a) - a R^*(y/a) - \langle x - y, \grad R^*(y/a) \rangle \\
& =
 a \left[ \langle x/a, \grad R^*(x/a) \rangle - R(\grad R^*(x/a)) \right] \\
& \qquad - a \left[ \langle y/a, \grad R^*(y/a) \rangle - R(\grad R^*(y/a)) \right]
- \langle x - y, \grad R^*(y/a) \rangle \\
& =
\langle x, \grad R^*(x/a) - \grad R^*(y/a) \rangle - a R(\grad R^*(x/a))
+ a R(\grad R^*(y/a)) \; .
\end{align*}
Part 2 of Proposition~\ref{proposition:conjugate-properties} implies that
\begin{align*}
u = \lim_{a \to 0^+} \grad R^*(x/a) & = \lim_{a \to 0^+} \argmin_{w \in K} \left( a R(w) - \langle x, w \rangle \right) \; , \\
v = \lim_{a \to 0^+} \grad R^*(y/a) & = \lim_{a \to 0^+} \argmin_{w \in K} \left( a R(w) - \langle y, w \rangle \right) \; .
\end{align*}
The limits on the right exist because of compactness of $K$. They are simply
the minimizers $u= \argmin_{w \in K} - \langle x, w \rangle$ and $v= \argmin_{w
\in K} - \langle y, w \rangle$ where ties in $\argmin$ are broken according to
smaller value of $R(w)$.

By assumption $R(w)$ is upper bounded. It is also lower bounded, since it is
defined on a compact set and it is lower semi-continuous. Thus,
\begin{align*}
& \lim_{a \to 0^+} a \Breg_{R^*}(x/a, y/a) \\
& = \lim_{a \to 0^+} \langle x, \grad R^*(x/a) - \grad R^*(y/a) \rangle - a R(\grad R^*(x/a)) + a R(\grad R^*(y/a)) \\
& = \lim_{a \to 0^+} \langle x, \grad R^*(x/a) - \grad R^*(y/a) \rangle = \langle x, u - v \rangle \; .
\end{align*}
\end{proof}

%% file: ada-ftrl-proofs.tex
\section{Proofs for AdaFTRL}
\label{section:ada-ftrl-proof}

\begin{proof}[Proof of Corollary~\ref{corollary:ada-ftrl-regret-bound}]
Let $S = \sup_{v \in K} f(v)$. Theorem~\ref{theorem:ada-ftrl-regret-bound}
applied to the regularizer $R(w) = \frac{c}{S} f(w)$ and
Proposition~\ref{proposition:diameter-vs-range} gives
$$
\Regret_T \le \sqrt{3}(1 + c) \max\left\{\sqrt{8}, \frac{1}{\sqrt{2c}} \right\} \sqrt{S \sum_{t=1}^T \|\ell_t\|_*^2} \; .
$$
It remains to find the minimum of $g(c) = \sqrt{3}(1 + c) \max\{\sqrt{8},
1/\sqrt{2c}\}$.  The function $g$ is strictly convex on $(0, \infty)$ and has
minimum at $c=1/16$ and $g(\frac{1}{16}) = \sqrt{3}(1+\frac{1}{16})\sqrt{8} \le
5.3$.
\end{proof}

\begin{proof}[Proof of Lemma~\ref{lemma:recurrence-solution}]
Let $a_t = \|\ell_t\|_* \max\{D, 1/\sqrt{2\lambda}\}$. The statement of the
lemma is equivalent to $\Delta_T \le \sqrt{3 \sum_{t=1}^T a_t^2}$ which we
prove by induction on $T$.  The base case $T=0$ is trivial. For $T \ge 1$, we
have
$$
\Delta_T
\le \Delta_{T-1} + \min \left\{a_T, \ \frac{a_T^2}{\Delta_{T-1}} \right\}
\le \sqrt{3 \sum_{t=1}^{T-1} a_t^2} + \min \left\{ a_T, \frac{a_T^2}{\sqrt{3 \sum_{t=1}^{T-1} a_t^2}} \right\}
$$
where the first inequality follows from Lemma~\ref{lemma:gap-recurrence}, and
the second inequality from the induction hypothesis and the fact that $f(x) = x
+ \min\{a_T, a_T^2/x\}$ is an increasing function of $x$.  It remains to prove
that
$$
\sqrt{3 \sum_{t=1}^{T-1} a_t^2} + \min \left\{ a_T, \frac{a_T^2}{\sqrt{3 \sum_{t=1}^{T-1} a_t^2}} \right\}
\le  \sqrt{3 \sum_{t=1}^T a_t^2} \; .
$$
Dividing through by $a_T$ and making substitution $z=\frac{\sqrt{\sum_{t=1}^{T-1} a_t^2}}{a_T}$, leads to
$$
z\sqrt{3} + \min\left\{1,\frac{1}{z\sqrt{3}}\right\} \le \sqrt{3 + 3z^2}
$$
which can be easily checked by considering separately the cases $z \in
[0,\frac{1}{\sqrt{3}})$ and $z \in [\frac{1}{\sqrt{3}}, \infty)$.
\end{proof}

%% file: solo-ftrl-proofs.tex
\section{Proofs for SOLO FTRL}
\label{section:solo-ftrl-proof}

\begin{proof}[Proof of Corollary~\ref{corollary:regret-solo-ftrl-bounded-set}]
Let $S = \sup_{v \in K} f(v)$. Theorem~\ref{theorem:regret-solo-ftrl} applied
to the regularizer $R(w) = \frac{c}{\sqrt{S}} f(w)$, together with
Proposition~\ref{proposition:diameter-vs-range} and a crude bound
$\max_{t=1,2,\dots,T} \|\ell_t\|_* \le \sqrt{\sum_{t=1}^T \|\ell_t\|_*^2}$,
give
$$
\Regret_T \le \left(c + \frac{2.75}{c}  + 3.5\sqrt{8} \right) \sqrt{S \sum_{t=1}^{T} \norm{\ell_t}_*^2} \; .
$$
We choose $c$ by minimizing $g(c) = c + \frac{2.75}{c} + 3.5\sqrt{8}$. Clearly,
$g(c)$ has minimum at $c = \sqrt{2.75}$ and has minimal value $g(\sqrt{2.75}) =
2\sqrt{2.75} + 3.5\sqrt{8} \le 13.3$.
\end{proof}

%% file: lower-bound-proof.tex
\section{Lower Bound Proof}
\label{section:lower-bound-proof}

\begin{proof}[Proof of Theorem~\ref{theorem:simple-lower-bound}]
Pick $x,y \in K$ such that $\|x - y\| = D$. This is possible since $K$ is compact.
Since $\|x - y\| = \sup \{\langle \ell, x - y \rangle ~:~ \ell \in V^*, \|\ell\|_* = 1\}$
and the set $\{ \ell \in V^* ~:~ \|\ell\|_* = 1 \}$ is compact, there exists $\ell \in V^*$
such that
$$
\|\ell\|_* = 1 \qquad \text{and} \qquad \langle \ell, x - y \rangle = \|x - y\| = D \; .
$$
Let $Z_1, Z_2, \dots, Z_T$ be i.i.d. Rademacher variables, that is,
$\Pr[Z_t = +1] = \Pr[Z_t = -1] = 1/2$. Let $\ell_t = Z_t a_t \ell$.
Clearly, $\|\ell_t\|_* = a_t$. The lemma will be proved if we show that
(\ref{equation:simple-lower-bound}) holds with positive probability.
We show a stronger statement that the inequality holds in expectation, i.e.
$\Exp[\Regret_T] \ge \frac{D}{\sqrt{8}} \sqrt{\sum_{t=1}^T a_t^2}$. Indeed,
\begin{align*}
\Exp\left[ \Regret_T \right]
& \ge \Exp\left[ \sum_{t=1}^T \langle \ell_t, w_t \rangle \right] - \Exp\left[\min_{u \in \{x,y\}} \sum_{t=1}^T \langle \ell_t, u \rangle \right] \\
& = \Exp\left[ \sum_{t=1}^T Z_t a_t \langle \ell, w_t \rangle \right] + \Exp \left[\max_{u \in \{x,y\}} \sum_{t=1}^T -Z_t a_t \langle \ell, u \rangle \right]  \\
& = \Exp\left[ \max_{u \in \{x,y\}} \sum_{t=1}^T -Z_t a_t \langle \ell, u \rangle \right] = \Exp\left[ \max_{u \in \{x,y\}} \sum_{t=1}^T Z_t a_t \langle \ell, u \rangle \right]  \\
& = \frac{1}{2} \Exp\left[ \sum_{t=1}^T Z_t a_t \langle \ell, x + y \rangle \right]  + \frac{1}{2}\Exp\left[ \left|\sum_{t=1}^T Z_t a_t \langle \ell, x - y \rangle \right| \right] \\
& = \frac{D}{2}\Exp\left[ \left|\sum_{t=1}^T Z_t a_t \right| \right] \ge \frac{D}{\sqrt{8}} \sqrt{\sum_{t=1}^T a_t^2}
\end{align*}
where we used that $\Exp[Z_t] = 0$, the fact that distributions of $Z_t$ and
$-Z_t$ are the same, the formula $\max\{a,b\} = (a+b)/2 + |a-b|/2$, and
Khinchin's inequality in the last step (Lemma A.9 in
\cite{Cesa-Bianchi-Lugosi-2006}).
\end{proof}